\documentclass[11pt]{article}
\usepackage{indentfirst}
\usepackage{fullpage}
\usepackage[numbers]{natbib}

\usepackage{hyperref}
\usepackage{url}
\usepackage[utf8]{inputenc} 
\usepackage[T1]{fontenc}    
\usepackage{hyperref}       
\usepackage{url}            
\usepackage{booktabs}       
\usepackage{amsfonts}       
\usepackage{nicefrac}       
\usepackage{microtype}      

\usepackage{amsmath}
\usepackage{amsfonts}
\usepackage{lastpage}
\usepackage{graphicx}
\usepackage{subfigure}
\usepackage{multirow}
\usepackage{tabu}
\usepackage{lipsum}

\usepackage{multirow}
\usepackage{tabu}
\usepackage{lipsum}
\usepackage{essay-def}
\usepackage{caption}
\usepackage{appendix}
\usepackage{mathrsfs}
\usepackage{indentfirst}
\usepackage{natbib}

\usepackage{xcolor}
\definecolor{SchoolColor}{rgb}{0.6471, 0.1098, 0.1882} 
\definecolor{forestgreen}{rgb}{0.0, 0.27, 0.13}
\usepackage{tikz}

\usepackage{amsmath}
\usepackage{amssymb}
\usepackage{mathtools}
\usepackage{amsthm}
\usepackage{essay-def}
\usepackage{caption}
\usepackage{algorithm}
\usepackage{algorithmic}

\usepackage[capitalize,noabbrev]{cleveref}

\theoremstyle{plain}
\newtheorem{theorem}{Theorem}[section]
\newtheorem{proposition}[theorem]{Proposition}
\newtheorem{lemma}[theorem]{Lemma}

\theoremstyle{definition}
\newtheorem{definition}[theorem]{Definition}

\theoremstyle{remark}

\usepackage{todonotes}

\title{Randomized Geometric Algebra Methods for \\Convex Neural Networks}

\author{%
  Yifei Wang, Sungyoon Kim, Paul Chu, Indu Subramaniam, Mert Pilanci \\
  Department of Electrical Engineering\\
  Stanford University\\
  Stanford, CA 94305 \\
  \texttt{\{wangyf18, sykim777, chupaul, indu22, pilanci\}@stanford.edu} \\
}
\date{}

\begin{document}

\maketitle

\begin{abstract}

We introduce randomized algorithms to Clifford's Geometric Algebra, generalizing randomized linear algebra to hypercomplex vector spaces. This novel approach has many implications in machine learning, including training neural networks to global optimality via convex optimization. Additionally, we consider fine-tuning large language model (LLM) embeddings as a key application area, exploring the intersection of geometric algebra and modern AI techniques. In particular, we conduct a comparative analysis of the robustness of transfer learning via embeddings, such as OpenAI GPT models and BERT, using traditional methods versus our novel approach based on convex optimization. We test our convex optimization transfer learning method across a variety of case studies, employing different embeddings (GPT-4 and BERT embeddings) and different text classification datasets (IMDb, Amazon Polarity Dataset, and GLUE) with a range of hyperparameter settings. Our results demonstrate that convex optimization and geometric algebra not only enhances the performance of LLMs but also offers a more stable and reliable method of transfer learning via embeddings.

\end{abstract}


\section{Introduction}
In this paper, we propose randomized algorithms for Clifford (Geometric) Algebra and investigate applications in feature-based transfer learning via convex optimization. The approach is based on an observation from \cite{pilanci2023complexity}, showing that we can exactly characterize optimal weights of a neural network with ``generalized cross-products" of training data points. Moreover, it is known that two-layer neural networks have equivalent convex reformulations \cite{pilanci2020neural}, and we can train two-layer neural networks using convex optimization. The two facts imply that we can find the optimal parameters of a two-layer neural network very efficiently, as we know both the closed form of the parameters and a convex reformulation of it. One caveat is that for high-dimensional data, it is computationally inefficient to calculate the whole generalized cross-product, as it would need solving a large linear system. To mitigate the issue, we present a novel algorithm that involves randomized embeddings of the dataset to calculate the generalized cross-product more efficiently. For details, see \cref{sec:Theory}.

Transferring a language model to specific tasks has been a prominent approach to solving tasks in the field of NLP, especially after the remarkable success of unsupervised pre-trained large language models(LLMs), e.g. BERT \cite{devlin2018bert}. One approach in the literature is feature-based transfer learning. In this particular approach, we do not train the whole language model again to transfer it to downstream tasks: rather, we freeze the model and use the intermediate and last layer information as ``embeddings" and train a simple model that exploits knowledge from the embeddings. These embeddings, analogous to Word2Vec \cite{church2017word2vec}, are expected to have extracted useful information about the language during the pretraining phase, and can be used to solve various NLP tasks such as text classification, retrieval problems, question answering, etc. Using fixed embeddings is especially favorable when we don't have enough computational resources to execute the whole model again, as during the process of finetuning we have to at least calculate a forward pass of a LLM. Also, it is favorable when we don't have direct access to the language model itself, which is the case for various commercial LLMs such as GPT. 

We apply our proposed method in the field of text classification via embedding based transfer learning. Across various benchmarks, we show that the proposed method has multiple benefits compared to training with the state-of-the art optimizer AdamW: (i) the method is much faster than existing optimization algorithms that use AdamW, (ii) the method reaches better training accuracy for various tasks, (iii) the method can achieve best test accuracy for some datasets, and (iv) the method is more robust to initialization or hyperparameters. The results show that the proposed method can be an attractive substitute for existing local optimization algorithms, especially when training simple models with low computation.

The paper is organized as follows: in Section 2, we discuss relevant backgrounds on the topic, such as text classification, transfer learning, and learning with geometric algebra. In Section 3, we introduce the theory behind the proposed method, mainly convex reformulation of neural networks and algorithms using generalized cross-products. In Section 4, we show the application of our method to various text classification tasks.

\section{Related works}

\subsection{Feature-based transfer learning for NLP}

Exploiting pre-trained features has been a classical approach for solving NLP tasks, and there is an extensive line of work to obtain meaningful features for language models that can be universally used for downstream tasks. Some representative approaches are \cite{bengio2000neural}, where they use a mapping of each word to a continuous embedding space to train an n-gram model, word2vec \cite{church2017word2vec}, ELMo \cite{peters-etal-2018-deep}, effective sentence embeddings \cite{arora2017simple, reimers2019sentence}, and contrastive learning approaches \cite{gao2021simcse, chuang2022diffcse, kim2021self}. 

In general, freezing the source model and using the outputs for downstream tasks is referred to as linear probing \cite{alain2016understanding}. While the original concept of probing is training a simple model(such as a linear model) on specifically designed tasks to understand the role of intermediate layers, it is also used as a term to denote a training scheme that freezes the model and uses simple models to further train on downstream tasks \cite{evci2022head2toe}. While the performance of linear probing is not as good as finetuning the whole model, it can be considered as an alternative approach when we don't have enough computation power or access to the model itself. 

With recent advances in language models \cite{devlin2018bert}, several works have attempted to use the intermediate/final outputs of LLMs to train a network for downstream tasks. The two major applications are: generating high-fidelity images or audio from a given text prompt \cite{chang2023muse, koh2023generating, ghosal2023text}, using the embeddings for dense retrieval or search algorithms \cite{peng2023gpt, lin2023vector, muennighoff2022sgpt}, and utilizing the embeddings to guide RL agents \cite{du2023guiding}.  

\subsection{Geometric algebra and language modeling}
%


In neural networks, geometric algebra is typically applied by substituting traditional input vectors and linear maps with multivectors from Geometric Algebra (GA) \cite{bayro2001geometric}. These hypercomplex algebras are capable of capturing symmetries effectively \cite{ruhe2023geometric, chang2023muse, brehmer2024geometric} and has natural hierarchies \cite{mani2023representing,ruhe2023geometric}, using them as an alternative to existing neural networks have certain benefits. Most recently, several lines of work \citep{brehmer2024geometric, de2024euclidean} utilized geometric algebra to build a transformer architecture for geometric data with symmetries. 

In the field of NLP, \cite{mani2023representing} uses GA to generate word embeddings. They claim that using them as word embeddings can lead to better performance due to the natural hierarchy of multivectors in GA and its complexity compared to simple vector algebra. 

\section{Neural network training via Geometric Algebra}
\label{sec:Theory}
\subsection{Convex reformulation of neural networks}

Suppose that $X\in\mbR^{n\times d}$ is the training data matrix and $y\in\mbR^n$ is the label vector. We primarily focus on the two-layer neural network architectures with ReLU activation given as
$$
f_{\theta,b}^\text{ReLU}(x) = (x^TW_1+b^T)_+w_2=\sum_{i=1}^m (x^Tw_{1,i}+b_i)_+w_{2,i},\quad \theta=(W_1,w_2), 
$$ where $W_1\in\mbR^{d\times m}, w_2\in\mbR^m$ are trainable weights, $b\in\mbR^m$ is the bias vector and the gated ReLU activation
$$
f_{\theta,b}^\text{ReLU}(x) = \sum_{i=1}^m (x^Tw_{1,i}+b_i)1[x^Th_i+c_i\geq 0]w_{2,i},\quad \theta=(W_1,w_2,H,c)
$$

with $W_1,H\in\mbR^{d\times m}, w_2,c\in\mbR^m$ are trainable weights, $b\in\mbR^m$ is the bias vector.

We also consider the $p$-norm based regularization of the network weights
$$
\mathcal{R}_p(\theta) = \frac{1}{2} (\lVert W_1 \rVert_p^2 + \lVert w_2 \rVert_p^2).
$$
For simplicity, we first consider the neural networks with bias-free neuron, i.e., $b=0$. We also write $f_{\theta}(X)=:f_{\theta,0}(X)$. Consider the following training problem of a two-layer neural network with ReLU activation:
\begin{equation}
    \min_{\theta} \ell(f^\text{ReLU}_{\theta}(X),y)+\beta\mathcal{R}_2(\theta),
\end{equation}
where $\beta>0$ is a regularization parameter and $l(\cdot,y)$ is a convex loss function. The convex optimization form of the above problem writes
\begin{equation}\label{cvxnn:relu}
\begin{aligned}
    \min_{\{(u_i,u_i')\}_{i=1}^p}&\; \ell\pp{\sum_{i \in \mathcal{I}}D_iX(u_i-u_i'),y}+\beta\sum_{i=1}^p(\|u_i\|_2+\|u_i'\|_2)\\
    \text{ s.t. }&(2D_i-I)Xu_i\geq 0, (2D_i-I)Xu_i'\geq 0.
\end{aligned}
\end{equation}
for some index set $\mathcal{I}$. Here $D_1,\dots,D_p$ are enumeration of all possible hyperplane arrangements $\{\diag(\mbI(Xw\geq 0))|, w\in\mbR^d\}$. When $\mathcal{I} = [p]$, we obtain the exact convex reformulation of the ReLU training problem, and solving the reformulation leads to solving the original nonconvex problem \cite{pilanci2020neural}. We may also consider the unconstrained version of the convex reformulation \cref{cvxnn:grelu}, which is equivalent to the gated ReLU activation \cite{mishkin2022fast}.
\begin{equation}\label{cvxnn:grelu}
\begin{aligned}
    \min_{\{(u_i)\}_{i=1}^p}&\; \ell\pp{\sum_{i=1}^pD_iXu_i,y}+\beta\sum_{i=1}^p\|u_i\|_2.
\end{aligned}
\end{equation}
An important feature of the convex optimization formulation \ref{cvxnn:grelu} is that it is a group Lasso problem, which can be efficiently solved by gradient-based algorithms like FISTA \cite{beck2009fast}. 

In practice, it is impossible to enumerate all possible hyperplane arrangements when $d$ is even moderately large. Instead, we subsample a subset of valid hyperplane arrangements. In this literature, it is common to use Gaussian sampling, where the patterns are given as $\bar D_i = \diag(\mbI(Xv_i\geq 0))$, and
$v_1,\dots,v_k$ are i.i.d. random vectors following $\mcN(0,I)$. Then, we solve the convex optimization problem with subsampled hyperplane arrangements.


\begin{algorithm}[!tp]
\caption{Convex neural network training via Gaussian sampling}
\label{alg:Gaussian}
\begin{algorithmic}[1]
\REQUIRE Number of hyperplane arrangement samples $k$, regularization parameter $\beta>0$.
\STATE Sample $k$ i.i.d. random vectors $v_1,\dots,v_k$ following $\mcN(0,I)$.
\STATE Compute $\bar D_i = \diag(\mbI(Xv_i\geq 0))$ for $i\in[k]$.
\STATE Solve the convex optimization problem \eqref{cvxnn:relu} with the subsampled patterns.
\end{algorithmic}
\end{algorithm}

A natural question is that whether there exists a more efficient way to sample the hyperplane arrangements.

\subsection{Geometric Algebra and neural networks}\label{sec:gann}
Clifford's Geometric Algebra is a mathematical framework that enables geometric objects of different dimensions to be expressed in a unified manner \citep{artin2016geometric}. Within its vast application in many different domains \citep{doran2003geometric,dorst2012applications}, we focus mainly on the optimal weights of a neural network and their functionality in the lens of geometric algebra. 

The geometric algebra over a $d$-dimensional Euclidean space is denoted as $\mbG^d$. Each element $M\in \mbG^d$ is a multivector which can be represented by
$$
M=\lra{M}_0+\lra{M}_1+\dots+\lra{M}_d.
$$
Here $\lra{M}_k$ denotes the $k$-vector part of $M$. A $k$-blade $M=\alpha_1\wedge \dots\wedge \alpha_k$ is a $k$-vector that can be expressed as the wedge product of $k$ vectors $\alpha_1,\dots,\alpha_k\in\mbR^d$. It can be viewed as a $k$-dimensional oriented parallelogram. For instance, $a\wedge b$ is a $2$-blade, which represents the signed area of the paralleogram spanned by $a$ and $b$. We can define the inner product between two $k$-blades $M=\alpha_1\wedge \dots\wedge \alpha_k$ and $N=\beta_1\wedge \dots\wedge \beta_k$ by $$
M\cdot N=|\{\alpha_i^T\beta_j\}_{i,j=1}^k|.
$$
where $|A|$ is the determinant of matrix $A$ and $\{\alpha_i^T\beta_j\}_{i,j=1}^k$ is a $k\times k$ matrix whose element at the $i$-th row and $j$-th column is defined by $\alpha_i^T\beta_j$. 
For each parallelogram represented by a $k$-blade, we can assign $(d-k)$-dimensional orthogonal complement. To be specific, for every pair of $k$-vectors $M,N\in\mbG^d$, there exists a unique $(d-k)$-vector $\star M\in\mbG^d$ such that
$$
\star M\wedge N = (M\cdot N) e_1\wedge\dots\wedge e_d = (M\cdot N) \mfI,
$$
where $\{e_i\}_{i=1}^d$ is the standard basis of $\mbR^d$ and $\mfI=: e_1\wedge \cdots \wedge e_d = e_1\cdots e_d$ stands for the unit pseudoscalar. This linear transform from $k$-vectors to $(d-k)$-vectors defined by $M\to \star M$ is the Hodge star operation, which satisfies $\star M = M \mfI^{-1}= M e_d\cdots d_1$. Based on the Hodge star operation, we can define the generalized cross-product in $\mbR^d$. It takes $d-1$ vectors $x_1,\dots,x_{d-1}$ and forms a vector which is orthogonal to all of them: 
$$
\times(x_1,\dots,x_{d-1}) \triangleq \star(x_1\wedge \dots \wedge x_{d-1}).
$$
To be precise, the generalized cross-product can be calculated as follows.
\begin{definition}
Let $x_1,\dots,x_{d-1}\in\mbR^d$ be a set of $d-1$ vectors and denote $A=\bmbm{x_1&\dots&x_{d-1}}$ as the matrix whose columns are the vectors $\{x_i\}_{i=1}^{d-1}$. The generalized cross-product of $\{x_i\}_{i=1}^{d-1}$ is defined as
\begin{equation}
\begin{aligned}
   \times(x_1,\dots,x_{d-1})    \triangleq &\sum_{i=1}(-1)^{i-1}|A_i|e_i,
\end{aligned}
\end{equation}
where $|A_i|$ is the determinant of the square matrix $A_i$, $A_i$ is the square matrix obtained from $A$ by deleting its $i$-th row.
\end{definition}
The cross product and the wedge product are related via the formula
\begin{equation}\label{equ:wedge}
    x^T\times(x_1,\dots,x_{d-1})=  
    \textbf{Vol}(\mcP(x,x_1,\dots,x_{d-1}))=
    (x\wedge x_1 \wedge \dots \wedge x_{d-1})\mathbf{I}^{-1},
\end{equation}
where $\mcP(x,x_1,\dots,x_{d-1})$ is the parallelotope spanned by vectors $\{x,x_1,\dots,x_{d-1}\}$, whose volume is given by the determinant $\mathbf{det}[x,x_1,...,x_d]$.

\cite{pilanci2023complexity} provides a geometric algebraic perspective on understanding how optimal weights are represented by the training data. Consider the following training problem of a two-layer ReLU neural network with $\ell_1$ regularization. 
\begin{equation}\label{train:noncvx}
    \min_{\theta} \ell\pp{f^\text{ReLU}_{\theta}(X),y}+\beta\mathcal{R}_1(\theta),
\end{equation}
where $\beta>0$ is a regularization parameter. The above problem is equivalent to the following convex Lasso problem
\begin{equation}\label{cvxnn:lasso}
    \min_{z}\ell\pp{Kz,y} + \beta \|z\|_1,
\end{equation}
where the dictionary matrix $K$ is defined by $K_{i,j}=\kappa(x_i,x_{j_1},\dots,x_{j_{d-1}})$ for a multi-index $j=(j_1,\dots,j_{d-1})$ which enumerates over all combinations of $d-1$ rows of $X\in\mbR^{n\times d}$  and
$$
\kappa(x,u_1,\dots,u_{d-1}) = \frac{\pp{x^T\times(u_1,\dots,u_{d-1})}_+}{\|\times(u_1,\dots,u_{d-1})\|_2}=\frac{\pp{\textbf{Vol}(\mcP(x,u_1,\dots,u_{d-1}))}_+}{\|\times(u_1,\dots,u_{d-1})\|_2}.
$$ 
Here we utilize the equation \eqref{equ:wedge}. From an optimal solution $z^*$ to \eqref{cvxnn:lasso}, an optimal ReLU neural network can be constructed as follows:
$$
f^\text{ReLU}_{\theta^*}(x)=\sum_{j=(j_1,\dots,j_{d-1})} z_j^*\kappa(x,x_{j_1},\dots,x_{j_{d-1}}).
$$
In other words, the optimal weights in \eqref{train:noncvx} can be found via a closed-form formula $\times(x_{j_1},\dots,x_{j_{d-1}})$, where $\{x_{j_i}\}_{i=1}^{d-1}$ is a subset of training data indexed by $j_1,\dots,j_{d-1}$ and $\times(x_{j_1},\dots,x_{j_{d-1}})$ is the generalized cross-product of $\{x_{j_i}\}_{i=1}^{d-1}$. In a geometric algebra perspective, this operation corresponds to first obtaining the volume of the parallelotope and dividing it with the base, which is equivalent to calculating a distance between each point $x$ and the span of $\{x_{j_1}, x_{j_2}, \cdots, x_{j_{d-1}}\}$.
As a corollary the hyperplane arrangement patterns of the optimal neural network should take the form:
\begin{equation}\label{sample:ga}
D =  \diag(\mbI(Xh\geq 0)),\quad h=\times (x_{j_1},\dots,x_{j_{d-1}}).  
\end{equation}
Thus, a strategy to subsample hyperplane arrangements via geometric algebra is to randomly sample a size-$(d-1)$ subset $i_1,\dots,i_{d-1}$ from $[n]$ and then compute $D$ via \eqref{sample:ga}, which is equivalent to randomly sampling a parallelotope based on the training data. 
In addition, we can find the full regularization path after subsampling or full enumeration. The following lemma is a simple consequence of methods known for the Lasso regularization path.
\begin{lemma}\label{lem:path}
The regularization path of the optimal solution to \eqref{train:noncvx} with respect to the regularization parameter $\beta>0$ can be calculated by solving \eqref{cvxnn:lasso}.
\end{lemma}
We provide an video illustration of the  regularization path of the optimal network (see the link in Section \ref{sec:num_toy}).

\subsection{Approximating generalized cross-product via sketching}
\label{section:sketch}

In practice, with large input dimension $d$, the computational cost of computing the generalized cross-product can be costly. To reduce the computation complexity in sampling optimal weight vectors via generalized cross-product, we perform randomized embeddings to reduce the input dimension. Formally, given a sketch size $r\ll d$ and an embedding matrix $S\in\mbR^{m\times d}$, we project the training data to dimension $r$, i.e., $XS$. With a proper choice of $S$, the random projection of the training data can approximately preserve pair-wise distance with high probability \citep{vempala2005random}. 

Regarding the choice of the sketching matrix $S$, we primarily focus on sparse Johnson-Lindenstrauss transform (SJLT) \citep{sjlt}, with one non-zero entry per column. 

Based on the sketching matrix $S$, we can compute the optimal weight vector for the projected dataset as follows:
\begin{equation}
    \tilde v= \times(Sx_{j_1},\dots,Sx_{j_{r-1}}).
\end{equation}
We then embed $\tilde v\in\mbR^r$ to $\mbR^d$ by $v=S^T\tilde v$. The approximate optimal weight vector $v$ has the following property: first, it is orthgonal to the data samples $x_{j_1},\dots,x_{j_{r-1}}$. 
\begin{proposition}
Let $\{j_i\}_{i=1}^{r-1}\subseteq [n]$ be a subset of $[n]$. Suppose that $v=S^T\times(Sx_{j_1},\dots,Sx_{j_{r-1}})$. Then, we have
\begin{equation}\label{ortho1}
    v^Tx_{j_i} = 0, \forall i\in[r-1].
\end{equation}
\end{proposition}
The following property shows that for any subsampled data $x_{j_r}$, the weight vector from randomized geometric algebra is orthogonal to it in expectation. 
\begin{proposition}
    Let $\{j_i\}_{i=1}^{r-1}\subseteq [n]$ be a subset of $[n]$ and $j_r\in [n]$. Suppose that $v=S^T\times(Sx_{j_1},\dots,Sx_{j_{r-1}})$. Assume that each row of $S$ follows the idential distribution. Then, we have
    \begin{equation}
        \mbE_S[v^Tx_{j_r}]=0.
    \end{equation}
\end{proposition}

We summarize the algorithm of training the convex neural network via randomized geometric algebra in Alg \ref{alg:GA}.
\begin{algorithm}[!htp]
\caption{Convex neural network training via randomized Geometric Algebra}
\label{alg:GA}
\begin{algorithmic}[1]
\REQUIRE Number of hyperplane arrangement samples $k$, regularization parameter $\beta>0$, sketching matrix $S\in\mbR^{m\times d}$.
\FOR{$i=1,\dots,k$}
\STATE Sample $\{j_i\}_{i=1}^{r-1}$ from $[n]$.
\STATE Compute $v_i=S^T\times(Sx_{j_1},\dots,Sx_{j_{r-1}})$.
\STATE Compute $\bar D_i = \diag(\mbI(Xv_i\geq 0))$.
\ENDFOR
\STATE Solve the convex optimization problem \eqref{cvxnn:relu} with subsampled arrangements.
\end{algorithmic}
\end{algorithm}
\subsection{Why use Geometric Algebra? Case analysis in 2D}
\label{section:WhyGA}
The intuition behind why using Geometric Algebra could lead to a better sampling algorithm than Gaussian sampling is as follows: When considering a Gaussian random matrix \(X\), the probability that a uniformly distributed point on the unit sphere falls within the chamber defined by \(Xh \geq 0\) is influenced by the Gaussian measure and given by
\[
\mathbb{P}_{h \sim \text{Unif}(\mathbb{S}^{d-1})}[1(Xh \geq 0)].
\]

where the chamber is defined as the set of directions that have the same arrangement pattern, i.e.,
\[
\mathcal{C}_{D} = \{s\ |\ 1(Xs \geq 0) = D,\ \lVert s \rVert_2 = 1 \}.
\]
In contrast, GA sampling is not influenced by the measure of the chambers as we illustrate in this section. Detailed proofs of the results in this section are available in Appendix \ref{app:section:WhyGA}.

Consider the case where $X$ is a random Gaussian matrix. We know that when we normalize the Euclidean length of each row, the rows of $X$ are marginally distributed uniformly on the unit sphere. In this case, when we sample $n$ points, it can be seen that the minimum chamber has probability of order $O(\frac{1}{n^2})$ under the Gaussian measure. We formalize this below.
\begin{theorem}
\label{prop:2D}
Suppose that $d=2$. Let $D_1, D_2, \cdots, D_n, \bar{D}_1, \bar{D}_2, \cdots, \bar{D}_n$ be $2n$ possible activation patterns, and $\bar{D}$ denotes the complement of $D$. Suppose $X$ is a Gaussian random matrix. Then, for $i \in [n]$, the joint distribution of probabilities 
$$
\mathbb{P}_{u \sim \mathbb{S}^{1}} [1(Xu \geq 0) = D_i],
$$
as a random variable of $X$ is distributed as
$$
\frac{1}{2} \frac{E_i}{\sum_{i=1}^{n} E_i},
$$
where $E_1, E_2, \cdots E_n$ is a sequence of i.i.d. exponential random variables. Also, with probability at least $1 - e^{-20} - \exp(-Cn)$, we have
$$
\min_{j \in [n]} \mathbb{P}_{u \sim \mathbb{S}^{1}} [1(Xu \geq 0) = D_j] = O(\frac{1}{n^2}),
$$
for some $C > 0$.
\end{theorem}

\cref{prop:2D} shows that when the distribution of the rows of $X$ is $Unif(\mathbb{S}^{1})$, the smallest chamber has volume scaling with $\frac{1}{n^2}$. Hence, when we apply Gaussian sampling, to guarantee that we have sampled all patterns, we need to sample at least $\Omega(n^2)$ times.\\
On the other hand, sampling with Geometric Algebra can sample all hyperplane arrangement patterns with high probability, only by sampling $O(n)$ patterns. That is because Geometric Algebra weighs the sampling probability of each activation pattern the same. Recall that in $\mathbb{R}^2$ and $\mathbb{G}^2$, the generalized cross product corresponds to rotating a vector 90 degrees.
\begin{theorem}
Suppose that $d=2$ and no two rows of $X$ are parallel. Consider the following instantiation of Geometric Algebra sampling:
\begin{enumerate}
\item Sample $i \in [n]$, and randomly rotate it 90 degrees, clockwise or counterclockwise. Let $v$ the obtained vector.
\item Compute $\diag(1(Xv \geq 0))$.
\end{enumerate}
Then, we have
$$
\mathbb{P}[diag(1(Xv \geq 0) = D_j)] = \frac{1}{2n}. 
$$
for all $j \in [2n]$.
\end{theorem}

We leave extending this analysis to high dimension to future work. However, the qualitative difference between the Geometric Algebra sampling and Gaussian sampling, where one is independent of the chamber measure and the other is dependent, remains the same in high dimensions. Hence, for higher dimensions, we can expect that Geometric Algebra sampling will lead to more efficient algorithms.





\newcommand{\urllink}{https://github.com/pilancilab/Randomized-Geometric-Algebra-Methods-for-Convex-Neural-Networks} 

\newcommand{\urllinkvideo}{https://github.com/pilancilab/Randomized-Geometric-Algebra-Methods-for-Convex-Neural-Networks/tree/main/video}

\section{Numerical Experiments}
We compare the convex neural network training with the bias term via Gaussian sampling (Alg. \ref{alg:Gaussian_bias}) and via randomized Geometric Algebra (Alg. \ref{alg:GA_bias}) along with directly training the non-convex neural network by optimizing \eqref{train:noncvx}. All numerical experiments are conducted on a Dell PowerEdge R840 workstation (64 core, 3TB ram). The code is available at \url{\urllink}. 
\subsection{Geometric Algebra vs Gaussian Sampling}\label{sec:num_toy}
We first test the performance of convex neural network training on a toy 2D spiral dataset with $n=160$ training data. For both convex training methods (with Gaussian sampling and Geometric Algebra sampling), we use $200$ hidden neurons and set $\beta=10^{-3}$. As the training dataset is 2-dimensional, we also enumerate all entries in the dictionary matrix $K$ in \eqref{cvxnn:lasso} and solve the convex lasso problem \eqref{cvxnn:lasso}. We also subsample $200$ rows from the dictionary matrix $K$ and solve the subsampled convex lasso problem as well. For the convex training method with Geometric Algebra and Gaussian Sampling, we subsample $200$ hyperplane arrangements and solve the convex optimization formulation \eqref{cvxnn:relu}. From Figure \ref{fig:spiral}, we note that convex training method with Geometric Algebra sampling is more capable of learning complicated decision regions from the training data compared to the one with Gaussian sampling. Via the convex lasso problem \eqref{cvxnn:lasso} and its subsampled version, we also animate the entire path of decision regions of the (subsampled) Convex Lasso method with respect to the regularization parameter $\beta$, which is available at \href{\urllinkvideo}{here}. 
\begin{figure}[t]
\centering
\begin{minipage}[t]{0.85\textwidth}
\centering
\includegraphics[width=\linewidth]{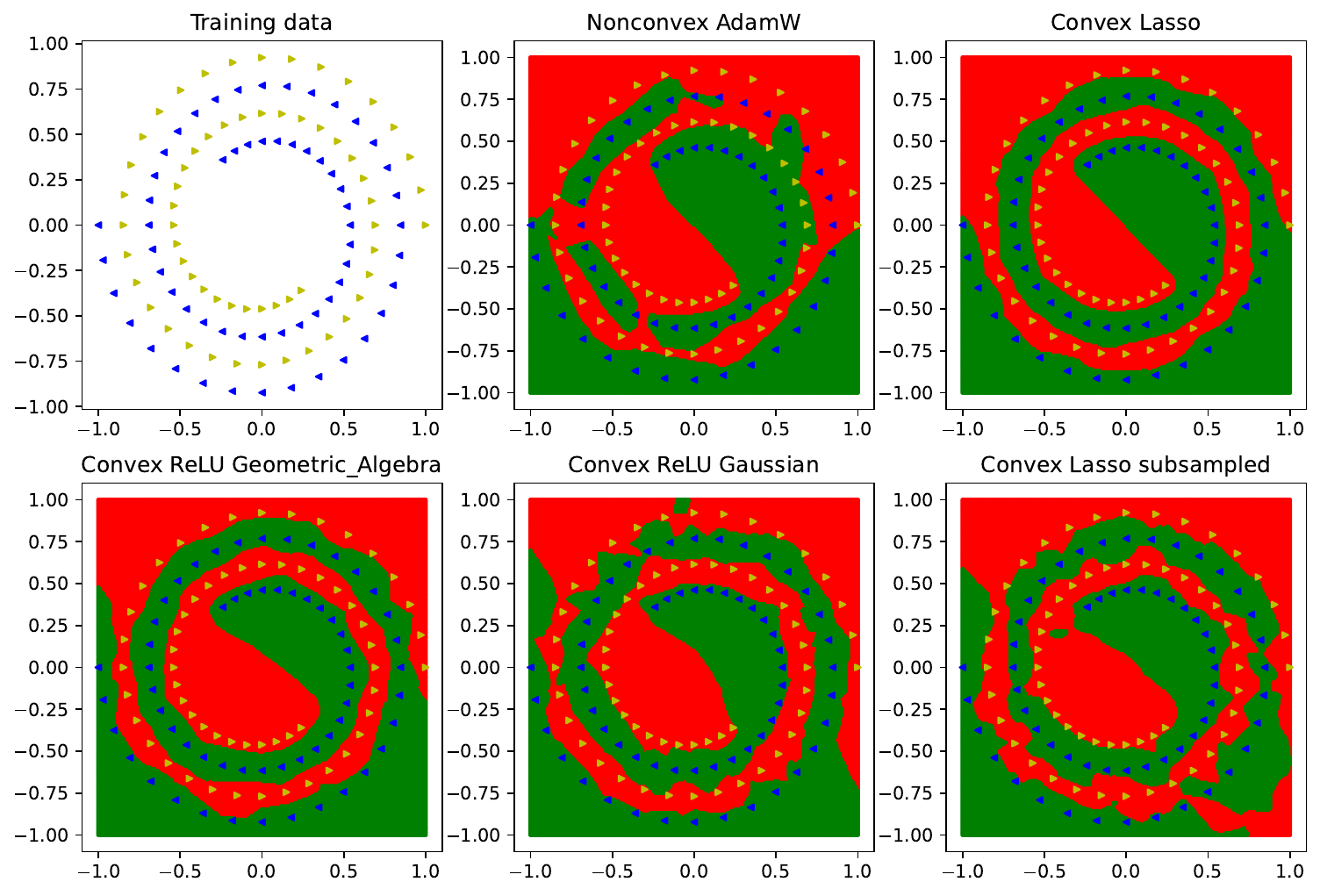}
\end{minipage}
\caption{Decision regions from different variants of convex optimization based training. The triangles represent data points in the training set. The \emph{Convex Lasso} method directly solves the convex lasso problem \eqref{cvxnn:lasso}. The \emph{Convex Lasso subsampled} method subsamples $200$ rows from the dictionary matrix $K$ in \eqref{cvxnn:lasso} and solves the subsampled problem. The methods  \emph{Geometric Algebra} and \emph{Gaussian} solve the convex optimization formulation \eqref{cvxnn:relu} with $200$ subsampled hyperplane arrangement patterns with geometric algebra and Gaussian samples respectively. See the video demonstration \href{https://anonymous.4open.science/r/CVXNN-randomized-GA-D400/video/full.mp4}{here}.}\label{fig:spiral}
\end{figure}

\subsection{Feature-based transfer learning}

We test upon the IMDb and GLUE-QQP datasets for sentimental analysis and ECG datasets with text/signal features for classification. For text datasets including IMDb and GLUE-QQP, we use OpenAI's test embedding model to extract feature vectors and train a neural network to classify the sentiment based on these features. We also compare with the baseline of a linear classifier based on extracted embedding features. The datasets we are employing are integral to our analysis, each offering various insights into how contextual information can significantly impact the performance of classification.

\begin{itemize}
    \item \textbf{IMDb Dataset:} This dataset is a collection of movie reviews from the IMDb website, designed for binary sentiment classification. With 25,000 training samples and an equal number of test samples, the dataset has been balanced across positive and negative reviews.

    \item \textbf{Amazon Polarity Dataset:} As a subset of a larger corpus, the Amazon Polarity dataset focuses on the sentiment aspect of customer reviews. It's a more extensive dataset, containing 3.6 million training samples and 400,000 test samples, categorized into positive and negative sentiments. We subsample $30000$ rows in our experiment.


    \item \textbf{GLUE CoLA Dataset (Corpus of Linguistic Acceptability):} This dataset is part of the GLUE benchmark, designed for the task of linguistic acceptability (judging whether a sentence is grammatically correct or not). It consists of sentences from professional linguistics literature, and it's typically used to assess the ability of models to understand the nuances of English grammar. The dataset contains around 10,000 sentences, split into training and test sets.

    \item \textbf{GLUE QQP Dataset (Quora Question Pairs):} Another component of the GLUE benchmark, this dataset is focused on determining whether a pair of questions asked on the Quora platform are semantically equivalent. It aims to foster the development of models that can understand and identify paraphrase in questions. The dataset is quite extensive, containing over 400,000 question pairs, with a balanced distribution of positive (paraphrase) and negative (non-paraphrase) examples. We subsample $50000$ rows in our experiment.

    \item \textbf{ECG Dataset:} The PTB-XL dataset is a 12-lead ECG waveforms dataset. It is stored in WFDB format at a 100Hz sampling rate. The data’s fidelity is ensured by 16-bit precision and 1$\mu V$/LSB resolution, with accompanying reports adhering to the SCPECG standard. There are over over 21,000 records from nearly 19,000 patients, with each record spanning 10 seconds. Cardiologists have provided multi-label annotations for each record, which are classified into major superclasses such as Normal ECG, Conduction Disturbance, Myocardial Infarction,	Hypertrophy and ST/T change. Here, we group this data into two major classes, Normal ECG and other 4 superclasses as Abnormal ECG, hence converting it into a binary classification task. 

     \item \textbf{MNIST Dataset:} The MNIST dataset is a standard benchmark  for evaluating the performance of image processing systems. It comprises 60,000 training images and 10,000 testing images of handwritten digits ranging from 0 to 9. We perform a binary classification over the class of $0$ and class of $1$.

\end{itemize}

%


For the non-convex training, we use the AdamW solver \citep{loshchilov2018fixing} and train the neural networks with $m=50$ neurons for $20$ epochs. For the learning rate, we perform a grid search upon $\{10^{-2},10^{-3}, 10^{-4}, 10^{-5}\}$ and choose the one with the best validation accuracy. For the convex training methods, we randomly sample $m=50$ hyperplane arrangement patterns via Gaussian sampling and randomized Geometric Algebra sampling respectively. The regularization parameter $\beta$ is chosen from $\{10^{-3}, 10^{-4}, 10^{-5}, 10^{-6}\}$ with the best validation accuracy. For convex neural network training with randomized Geometric Algebra, we use the SJLT matrix as the sketch matrix and set the sketch dimension $r=100$. For each compared method, we use different sizes of input training data ($n\in\{200,400,\dots,2000\}$) and plot the corresponding test accuracy. The standard deviation is calculated across $5$ independent trials. 

\newcommand{\figsize}{0.4}
\newcommand{\figsizepp}{0.3}

\begin{figure}[H]
\centering
\begin{minipage}[t]{\figsize\textwidth}
\centering
\includegraphics[width=\linewidth]{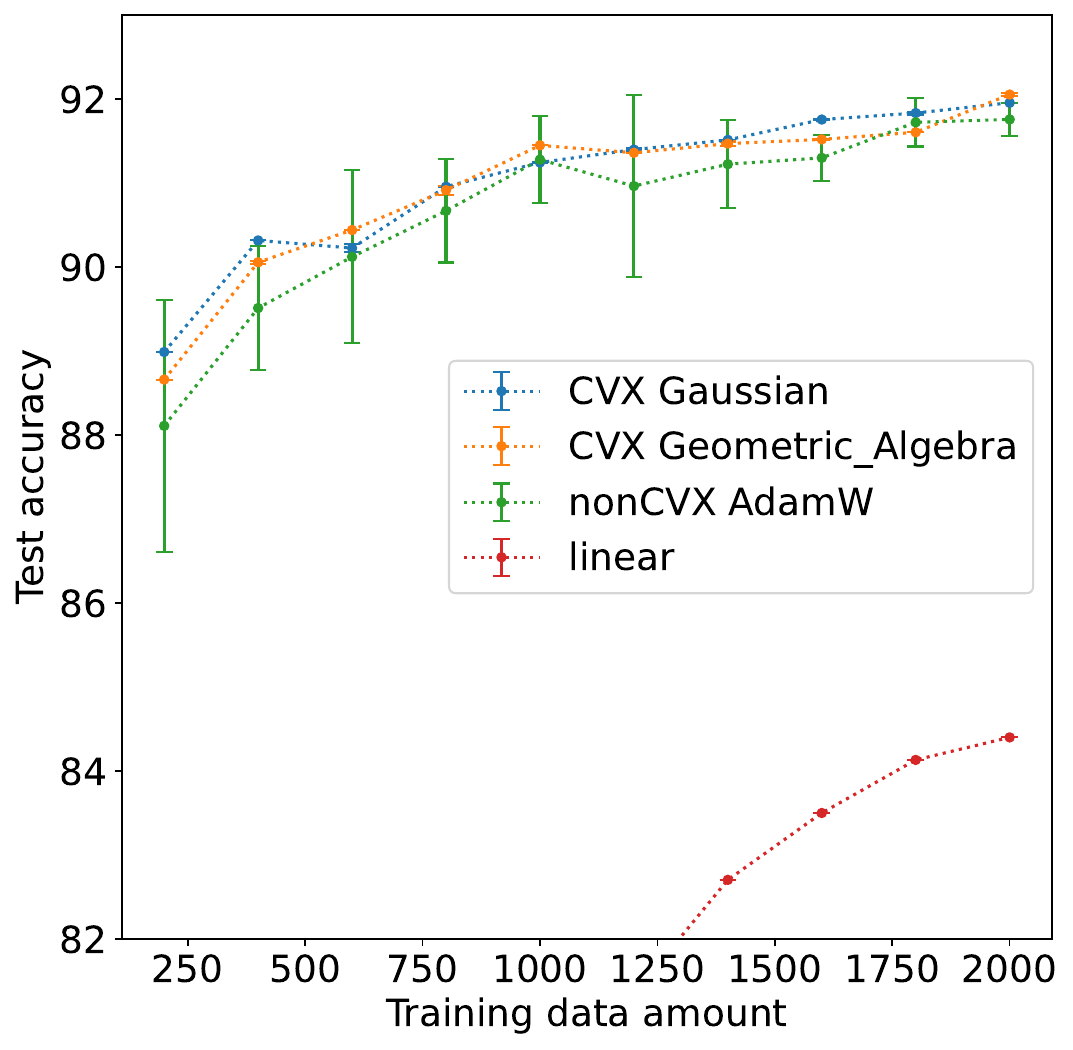}
\caption*{IMDB}
\end{minipage}
\begin{minipage}[t]{\figsize\textwidth}
\centering
\includegraphics[width=\linewidth]{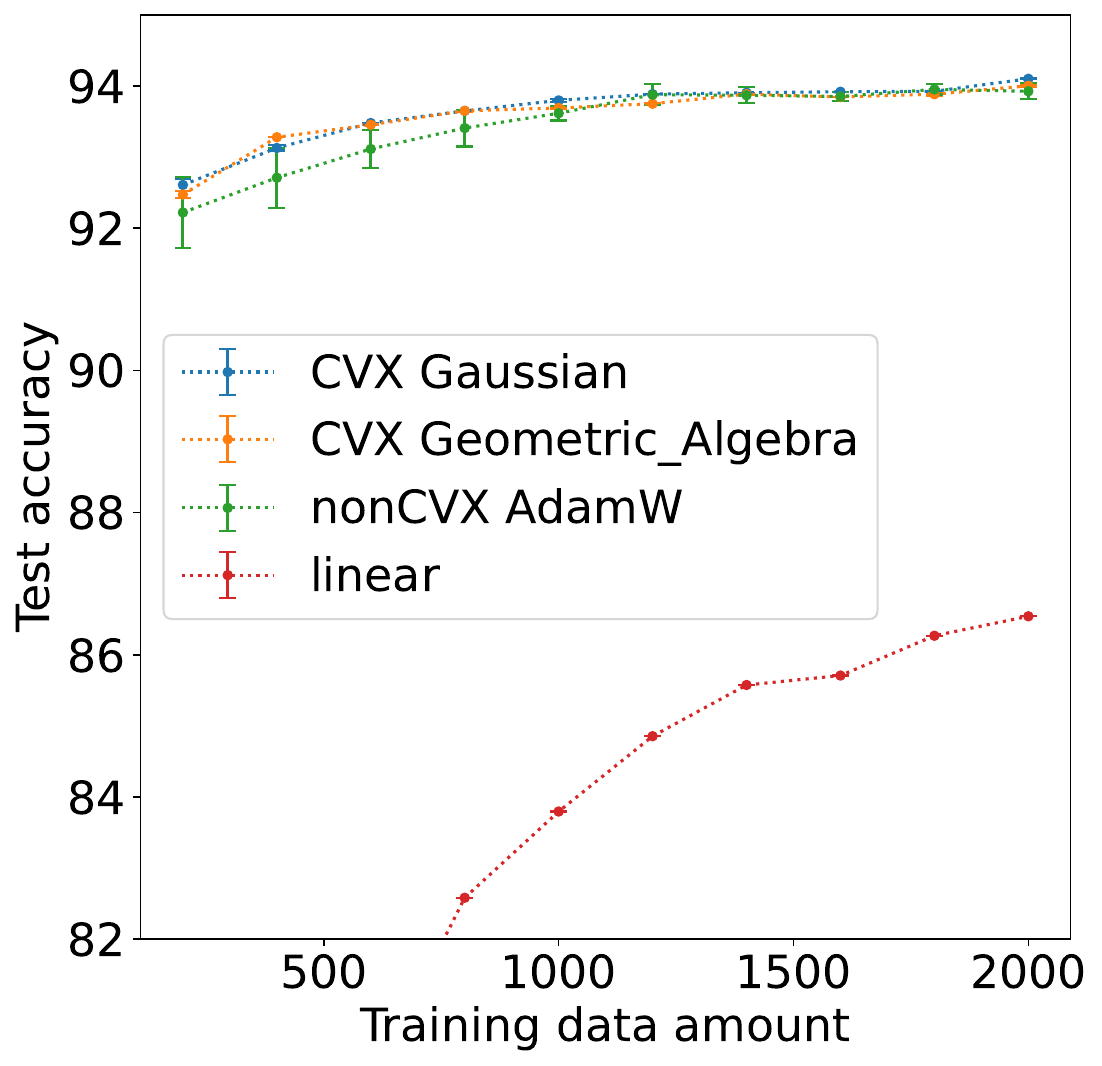}
\caption*{Amazon}
\end{minipage}
\begin{minipage}[t]{\figsize\textwidth}
\centering
\includegraphics[width=\linewidth]{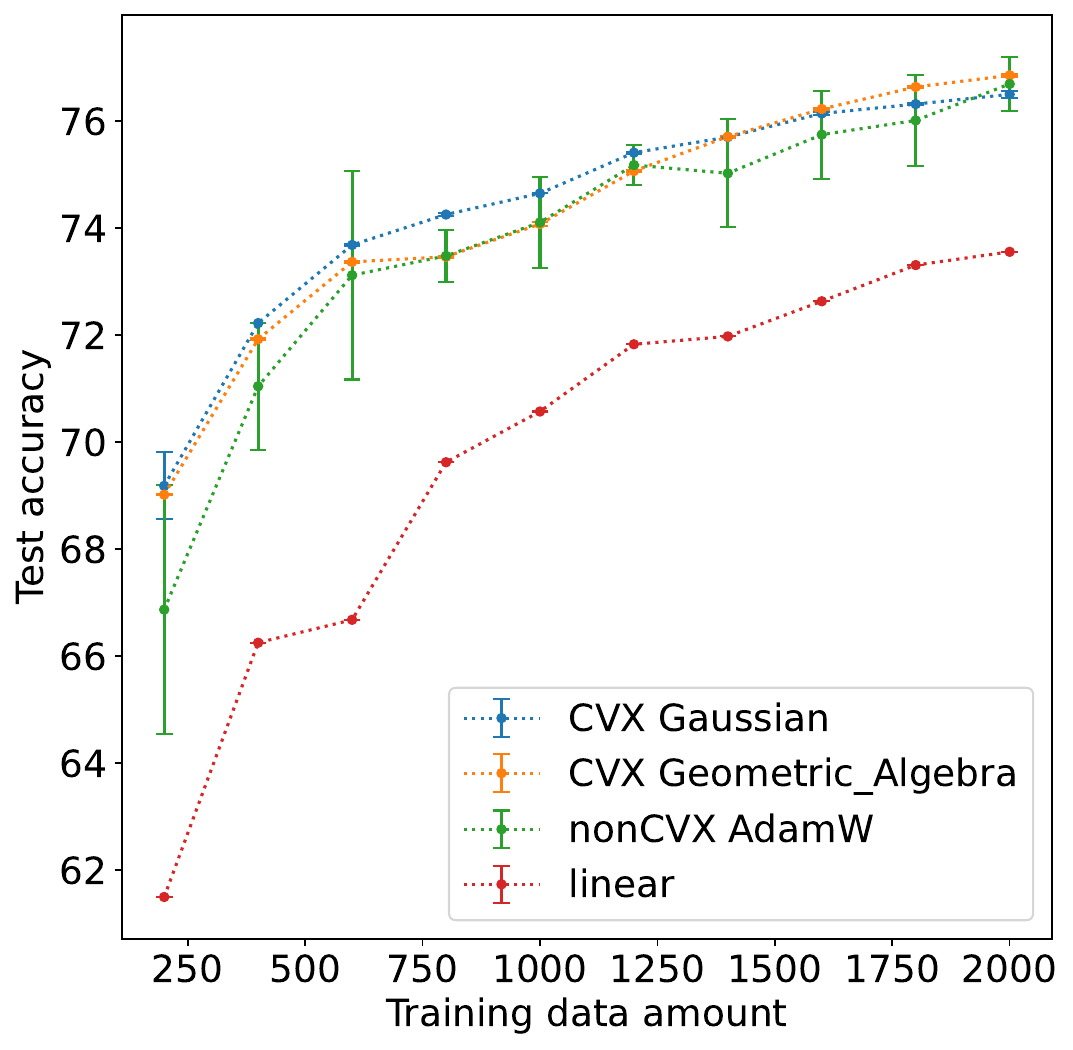}
\caption*{GLUE-QQP}
\end{minipage}
\begin{minipage}[t]{\figsize\textwidth}
\centering
\includegraphics[width=\linewidth]{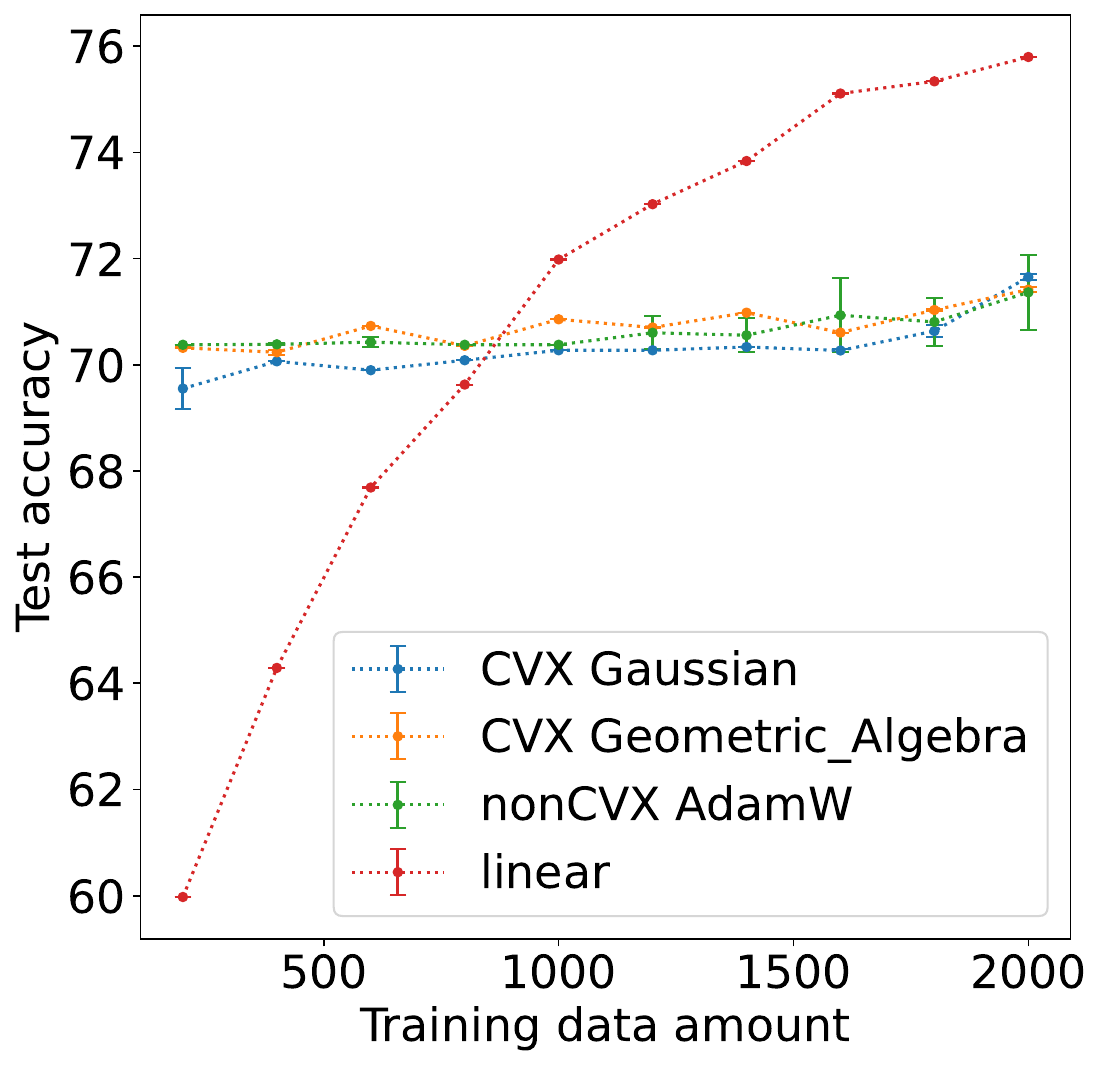}
\caption*{GLUE-COLA}
\end{minipage}
\centering
\begin{minipage}[t]{\figsizepp\textwidth}
\centering
\includegraphics[width=\linewidth]{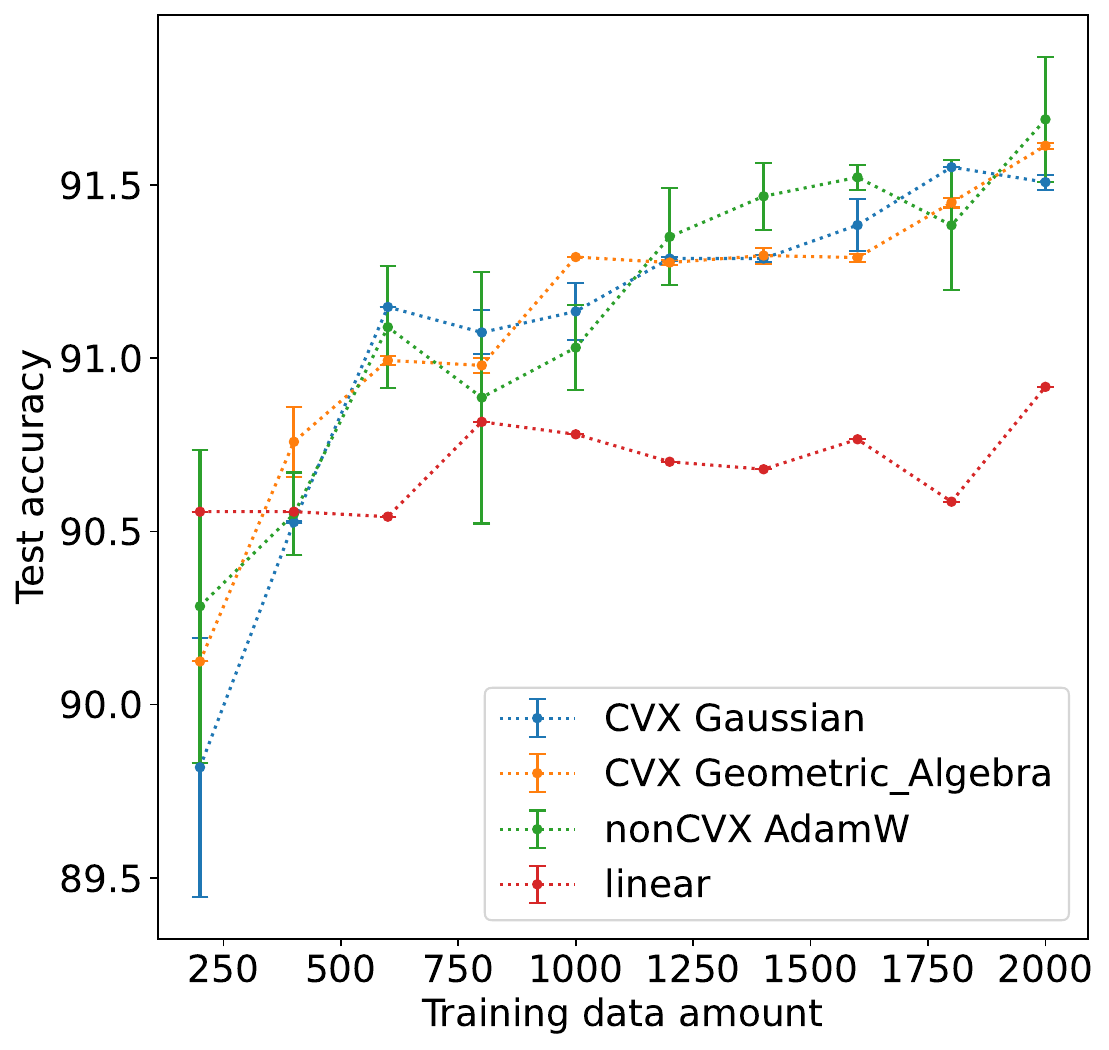}
\caption*{ECG-report}
\end{minipage}
\begin{minipage}[t]{\figsizepp\textwidth}
\centering
\includegraphics[width=\linewidth]{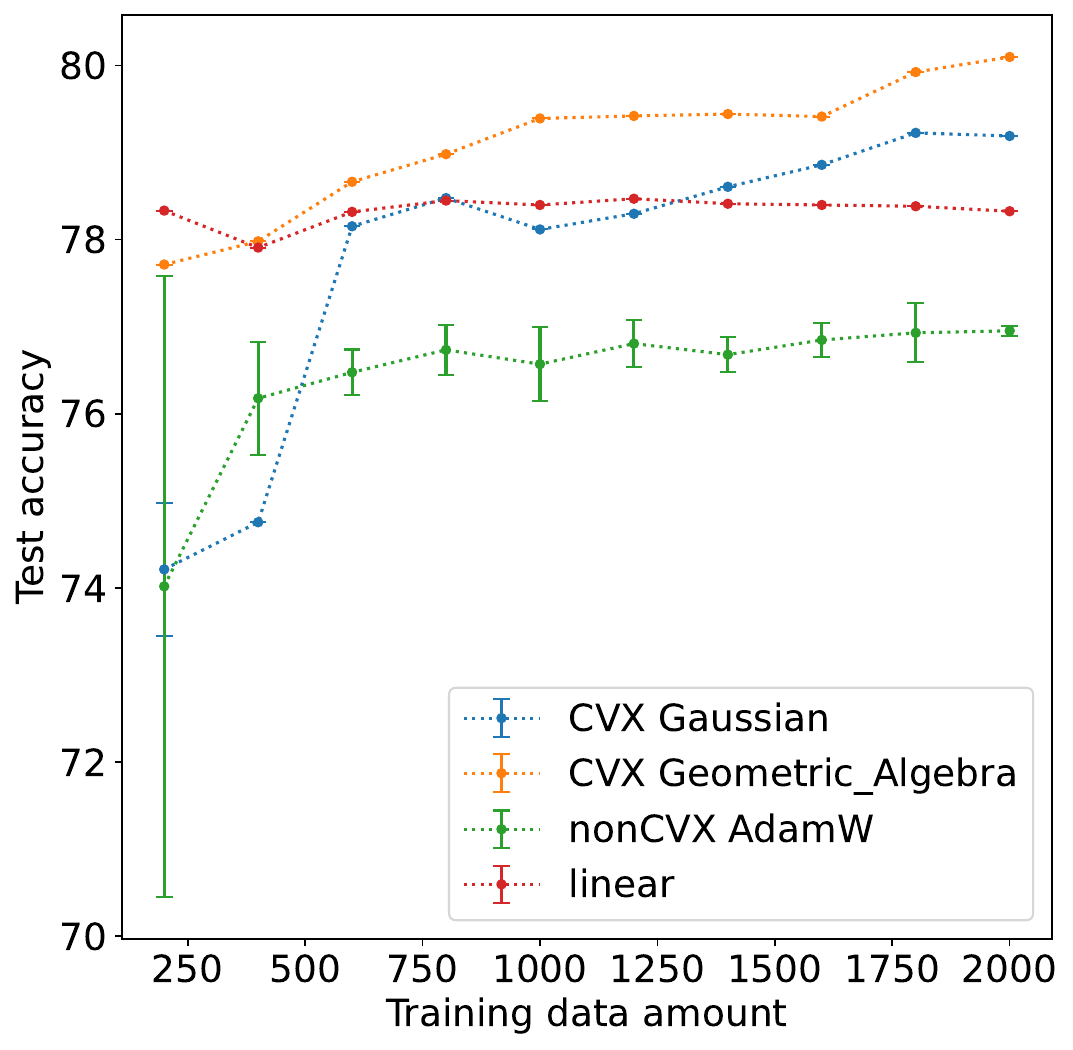}
\caption*{ECG-signal}
\end{minipage}
\begin{minipage}[t]{\figsizepp\textwidth}
\centering
\includegraphics[width=\linewidth]{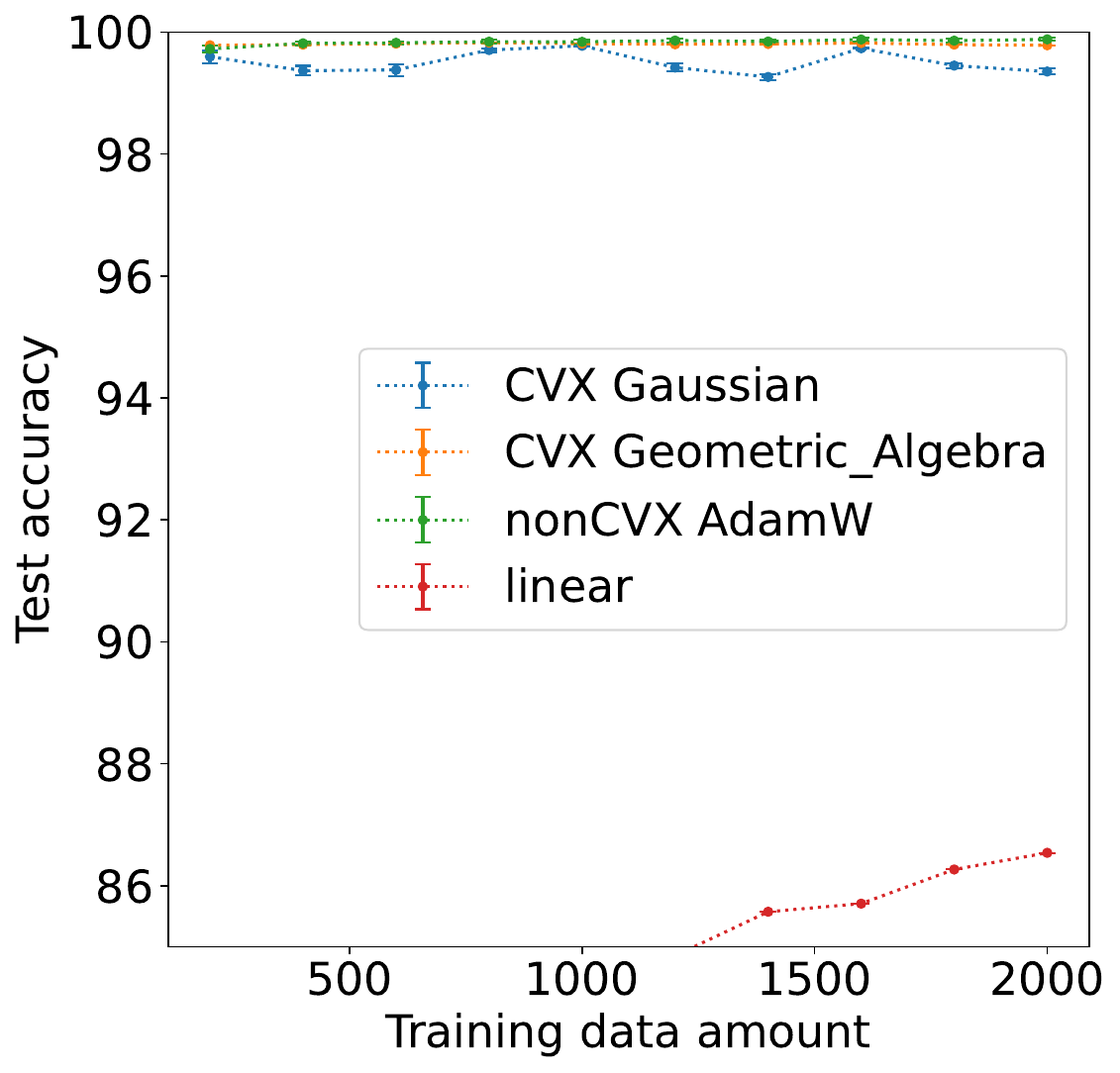}
\caption*{MNIST}
\end{minipage}
\caption{Test accuracy with different sizes of training data. $n\in\{200,400,\dots,2000\}$. The learning rates of AdamW and our method are chosen from a grid search over $\{10^{-2},10^{-3}, 10^{-4}, 10^{-5}\}$ according to the best validation accuracy.}
\end{figure}

Through experiments on various datasets, we show the efficiency of convex optimization methods. The standard deviations in the test accuracy of convex optimization methods are significantly smaller than the non-convex training method, especially when the amount of training data is limited. We also note that the two-layer neural network model significantly outperforms the linear classifier baseline. 

To illustrate the robustness and efficiency of our convex optimization method, we focus on the training dataset with size $n=2000$ and plot the curve of training/test accuracy with respect to the time. 

\newcommand{\figsizep}{0.85}

\begin{figure}[H]
\centering
\begin{minipage}[t]{\figsizep\textwidth}
\centering
\includegraphics[width=\linewidth]{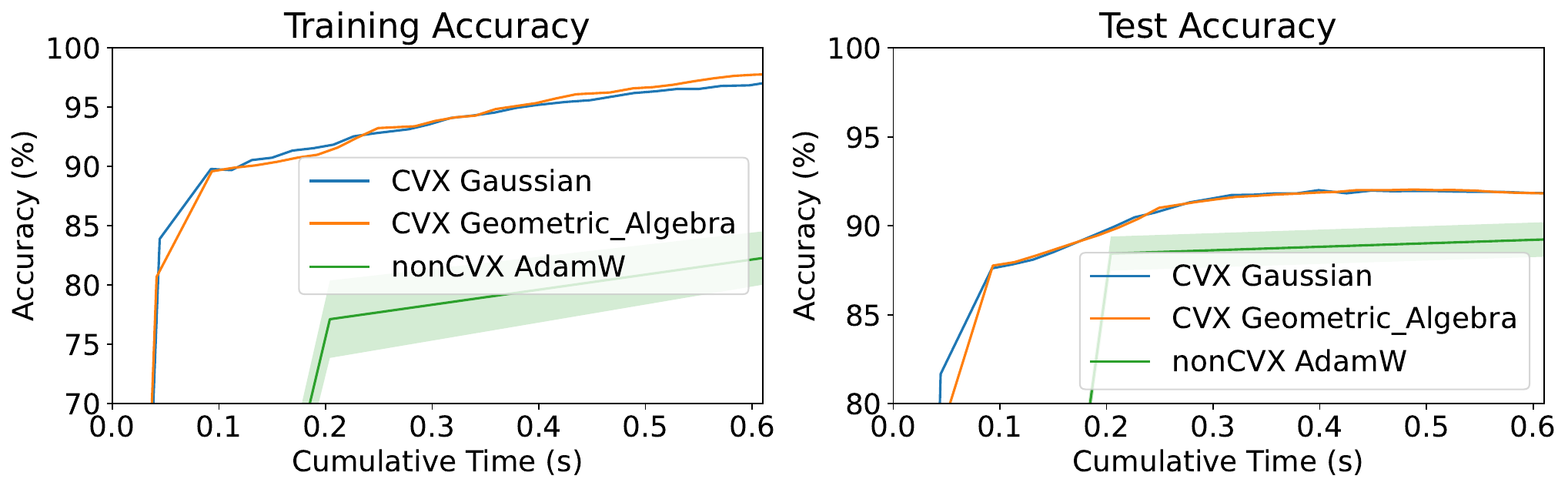}
\caption*{IMDB}
\end{minipage}
\begin{minipage}[t]{\figsizep\textwidth}
\centering
\includegraphics[width=\linewidth]{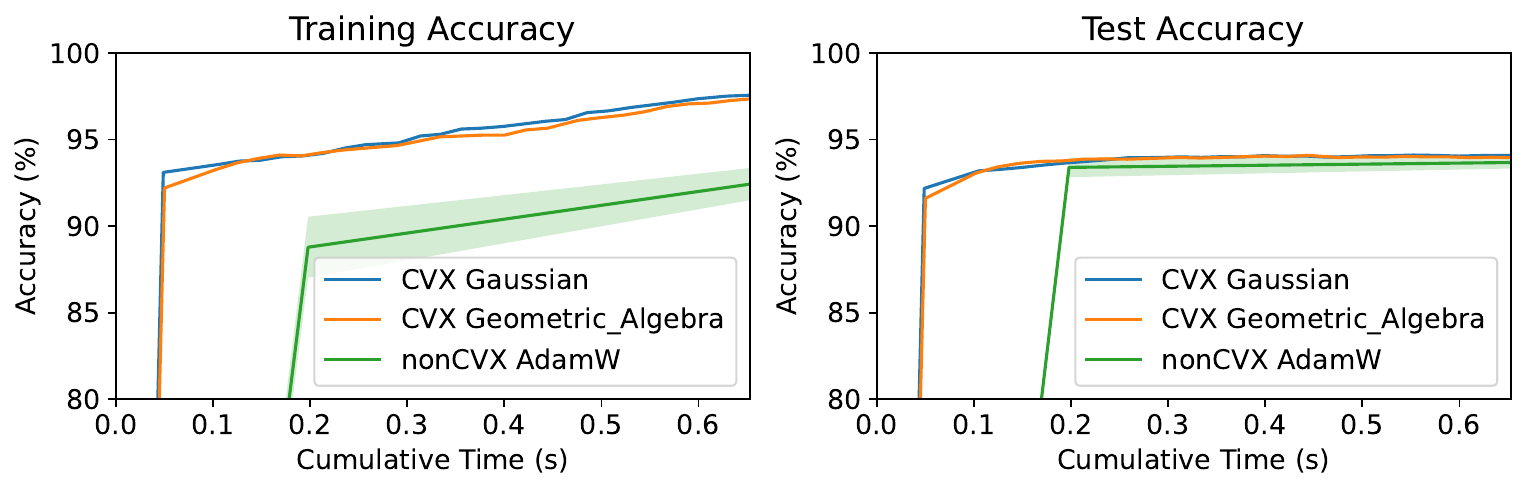}
\caption*{Amazon}
\end{minipage}
\begin{minipage}[t]{\figsizep\textwidth}
\centering
\includegraphics[width=\linewidth]{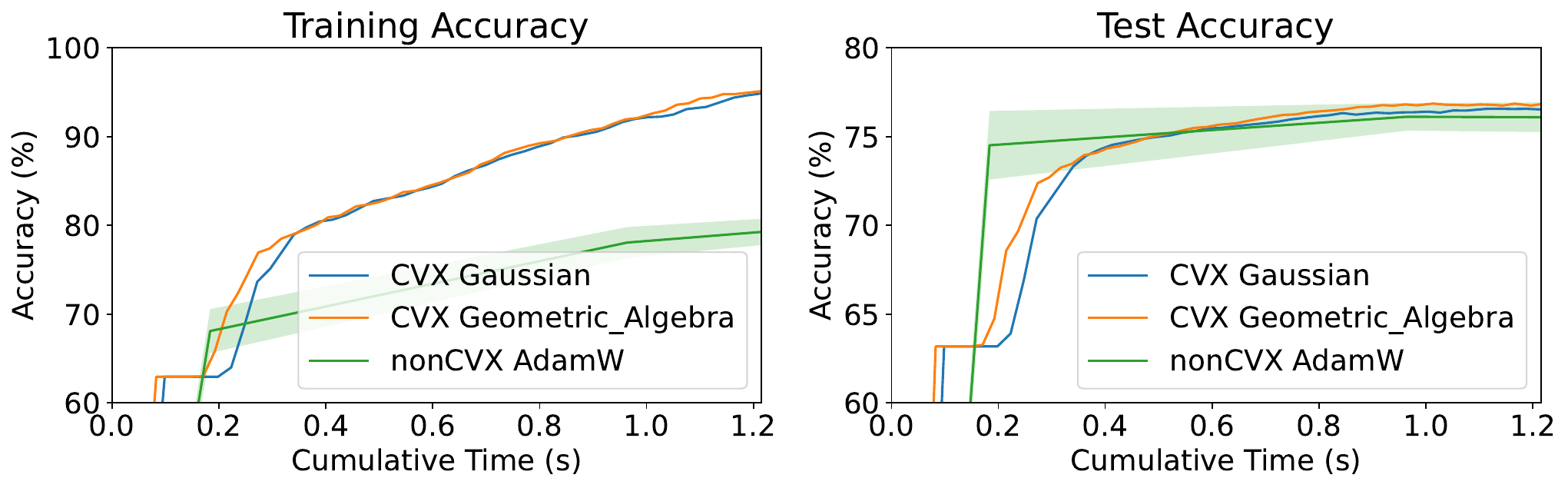}
\caption*{GLUE-QQP}
\end{minipage}
\caption{Train/test accuracy with respect to cpu time. The shaded area represents the standard deviation across 5 independent trials.}
\end{figure}

\begin{figure}[H]
\centering
\begin{minipage}[t]{\figsizep\textwidth}
\centering
\includegraphics[width=\linewidth]{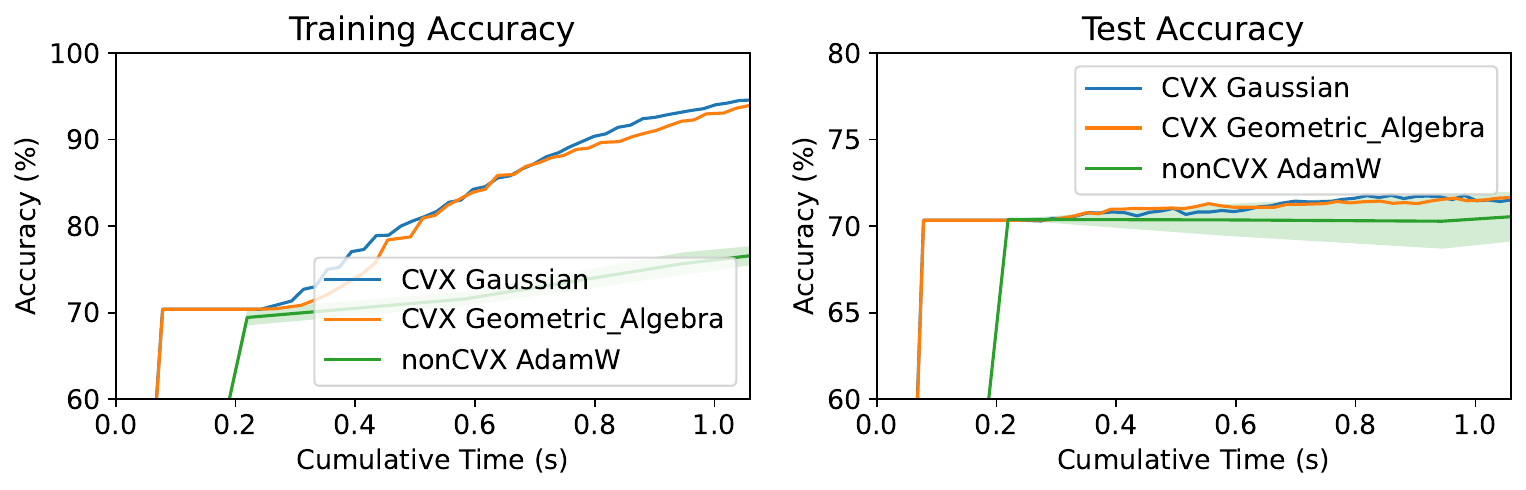}
\caption*{GLUE-COLA}
\end{minipage}
\begin{minipage}[t]{\figsizep\textwidth}
\centering
\includegraphics[width=\linewidth]{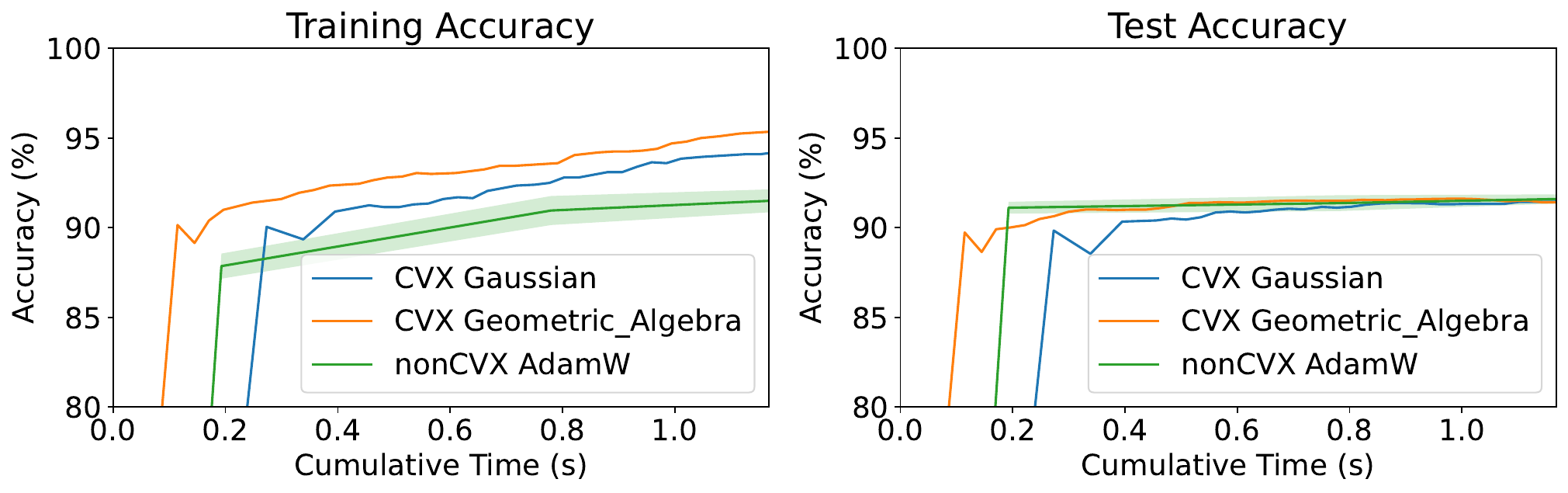}
\caption*{ECG-report.}
\end{minipage}
\begin{minipage}[t]{\figsizep\textwidth}
\centering
\includegraphics[width=\linewidth]{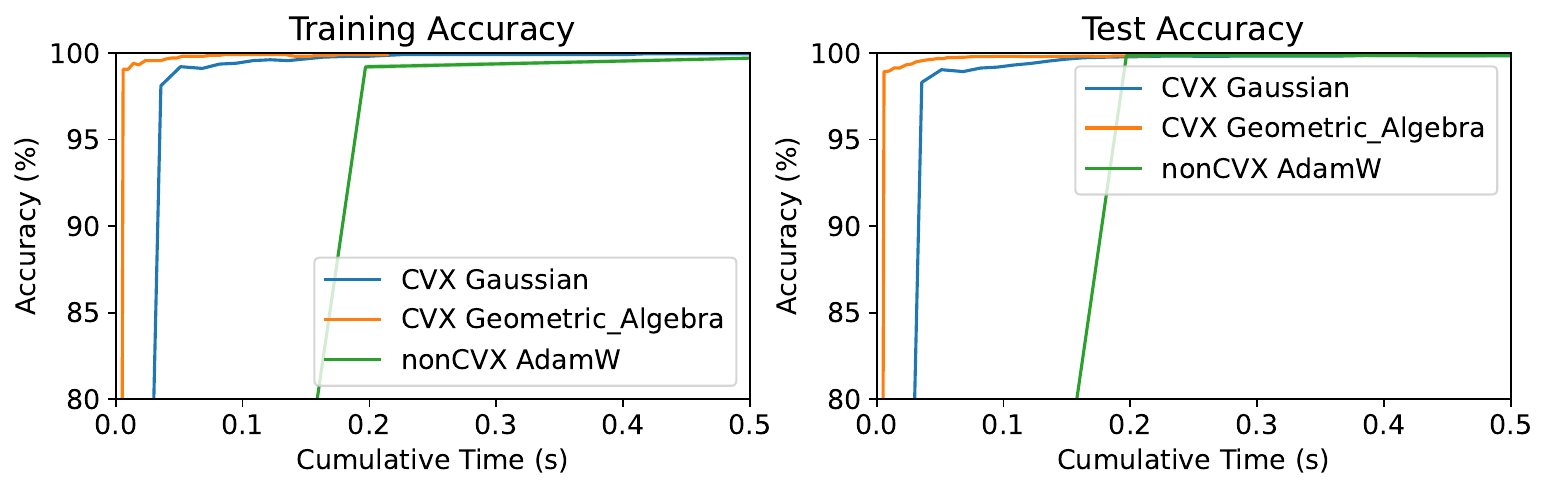}
\caption*{MNIST.}
\end{minipage}
\caption{Train/test accuracy with respect to cpu time. The shaded area represents the standard deviation across 5 independent trials.}
\end{figure}

Our results, indicated in orange, demonstrate the consistent and robust performance of the proposed approach. Our analysis reveals that models employing convex optimization not only perform well on speed but give also a decent accuracy boost. This characteristic is particularly notable if we are limited in terms of training data or computational power. The observed volatility in outcomes from non-convex optimization (AdamW), when varying the seeds, may be attributed to the convergence to different local minima or saddle points, as well as the influence of noise affecting the optimization process.


\section{Conclusion}

In this paper we introduce a new method to train convex neural networks based on geometric algebra. Inspired by the characterization of optimal weights, we propose a new randomized algorithm to sample hyperplane arrangement patterns of convex neural networks. Various experiments on transfer learning show that by obtaining patterns using the proposed method can improve both the training/ test accuracy and is more robust compared to non-convex counterparts. Our work could be extended to different transfer learning settings, and an initialization scheme based on Clifford Algebra may be effective.

\section*{Acknowledgements}
This work was supported in part by the National Science Foundation (NSF) under Grants ECCS-2037304 and DMS-2134248; in part by the NSF CAREER Award under Grant CCF-2236829; in part by the U.S. Army Research Office Early Career Award under Grant W911NF-21-1-0242; and in part by the Office of Naval Research under Grant N00014-24-1-2164.

\bibliography{Newton}
\bibliographystyle{apalike}

\newpage
\appendix
\onecolumn
\section{Proofs from \cref{sec:gann}}
\subsection{Proof of Lemma \ref{lem:path}}
\begin{proof}
For each regularization parameter $\beta>0$, the optimal solution of \eqref{train:noncvx} takes the form 
$$
f^\text{ReLU}_{\theta^*}(x)=\sum_{j=(j_1,\dots,j_{d-1})} z_j^*\kappa(x,x_{j_1},\dots,x_{j_{d-1}}).
$$
where $z^*$ is the optimal solution to \eqref{cvxnn:lasso}. This implies that it is sufficient to use the regularization path to \eqref{cvxnn:lasso} to characterize the regularization path of the optimal solution to \eqref{train:noncvx} with different regularization parameter $\beta$. Moreover, the regularization path to  \eqref{cvxnn:lasso} can be computed via the LARS algorithm \citep{tibshirani2013lasso}, and it will terminate with at most $3^q$ iteration. Here $q$ is the number of columns in the dictionary matrix $K$, which is upper bounded by $q\leq 2\binom{n}{d-1}$.
\end{proof}
\section{Proofs from \cref{section:sketch}}
\begin{proposition}
Let $\{j_i\}_{i=1}^{r-1}\subseteq [n]$ be a subset of $[n]$. Suppose that $v=S^T\times(Sx_{j_1},\dots,Sx_{j_{r-1}})$. Then, we have
\begin{equation}
    v^Tx_{j_i} = 0, \forall i\in[r-1].
\end{equation}
\end{proposition}
\begin{proof}
Let $\tilde v = \times(Sx_{j_1},\dots,Sx_{j_{r-1}})$. From the property of generalized cross-product, we have 
\begin{equation}
    (Sx_{j_i})^T\tilde v = 0, \forall i\in[r-1].
\end{equation}
This proves \eqref{ortho1}. 
\end{proof}
\begin{proposition}
    Let $\{j_i\}_{i=1}^{r-1}\subseteq [n]$ be a subset of $[n]$ and $j_r\in [n]$. Suppose that $v=S^T\times(Sx_{j_1},\dots,Sx_{j_{r-1}})$. Assume that each row of $S$ follows the idential distribution. Then, we have
    \begin{equation}
        \mbE_S[v^Tx_{j_r}]=0.
    \end{equation}
\end{proposition}
\begin{proof}
Let $A=\bmbm{x_{j_1},\dots,x_{j_r}}$. From the property of the generalized cross-product, we note that
\begin{equation}
    x_{j_r}^Tv = (Sx_{j_r})^T\times(Sx_{j_1},\dots,Sx_{j_{r-1}}) = |SA|. 
\end{equation}
Let $\bar S$ be the matrix obtained from $S$ by exchanging the first two rows of $S$. From the assumption, $\bar S$ and $S$ follow the same distribution. Therefore, we can compute that
\begin{equation}
    \mbE_S[v^Tx_{j_r}]=\mbE_S[|SA|] = \mbE_{\bar S}[|\bar SA|] = -\mbE_S[|SA|].
\end{equation}
This implies that $\mbE_S[v^Tx_{j_r}]=\mbE_S[|SA|]=0$. 
\end{proof}

\section{Proofs from \cref{section:WhyGA}}\label{app:section:WhyGA}
\begin{theorem}
Let $D_1, D_2, \cdots, D_n, \bar{D}_1, \bar{D}_2, \cdots, \bar{D}_n$ be $2n$ possible activation patterns, and $\bar{D}$ denotes the complement of $D$. Suppose $X$ is a Gaussian random matrix. Then, for $i \in [n]$, the joint distribution of probabilities 
$$
\mathbb{P}_{u \sim \mathbb{S}^{1}} [1(Xu \geq 0) = D_i],
$$
as a random variable of $X$ is distributed as
$$
\frac{1}{2} \frac{E_i}{\sum_{i=1}^{n} E_i},
$$
where $E_1, E_2, \cdots E_n$ is a sequence of i.i.d. exponential random variables. Also, with probability at least $1 - e^{-20} - \exp(-Cn)$,
$$
\min_{i \in [n]} \mathbb{P}_{u \sim \mathbb{S}^{1}} [1(Xu \geq 0) = D_j] = O(\frac{1}{n^2}),
$$
for some $C > 0$.
\end{theorem}
\begin{proof}
The first part of the proof follows directly from \cite{Devr86}, and from the fact that the clockwise angle (angle measured at clockwise orientation from the positive $x$ axis) is uniformly distributed on the interval $[0, 2\pi]$. The second part of the proof follows from the order statistics of exponential variables. We first know that
$$
\min_{i \in [n]} \mathbb{P}_{u \sim \mathbb{S}^{1}} [1(Xu \geq 0) = D_i] = \frac{1}{2} \frac{E_{(n)}}{\sum_{i=1}^{n} E_{(i)}},
$$
$E_{(n)} < E_{(n-1)} < \cdots E_{(1)}$ are the order statistics of $n$ i.i.d. samplings of Exp(1). Then,
$$
\mathbb{P}[E_{(n)} \leq \frac{20}{n}] = 1 - (e^{-20/n})^{n} = 1 - e^{-20}.
$$
Also, we know that from the concentration of exponential random variables, there exists a constant $C > 0$ that satisfies
$$
\mathbb{P}[0.5n \leq \sum_{i=1}^{n} E_i \leq 1.5n] \geq 1 - \exp(-Cn).
$$
In both cases, randomness comes from drawing $n$ i.i.d. samples from the exponential distribution, which is equivalent to drawing $2n$ chamber sizes from a uniform distribution of $n$ points. Combining the two high-probability bounds leads to the wanted result. 
\end{proof}

\begin{theorem}
Suppose no two rows of $X$ are parellel. Consider the following instantiation of Geometric Algebra sampling:\\
\qquad (1) sample $i \in [n]$, and randomly rotate it 90 degrees, clockwise or counterclockwise. Let $v$ the obtained vector.\\
\qquad (2) compute $D_i = diag(1(Xv \geq 0))$. \\
Then, we have
$$
\mathbb{P}[diag(1(Xv \geq 0) = D_j)] = \frac{1}{2n}. 
$$
for all $j \in [2n]$.
\end{theorem}
\begin{proof}
 We know that each chamber $\mathcal{C}_{D} := \{v\ |\ 1(Xv \geq 0) = D\}$ has a unique vector $u_D \in \{R_{\pi/2}(X_i), R_{-\pi/2}(X_i)\}_{i=1}^{n}$ that satisfies $1(Xu_D \geq 0) = D$. The reason is for some $u \in \mathcal{C}_D$, when we rotate $u$ clockwise until we meet a vector $u_0 \in \{R_{\pi/2}(X_i), R_{-\pi/2}(X_i)\}_{i=1}^{n}$, $1(Xu_0 \geq 0) = D$ should hold. Hence, there exists a one-to-one correspondence between the outer products $\{R_{\pi/2}(X_i), R_{-\pi/2}(X_i)\}_{i=1}^{n}$ and the activation patterns. As the geometric algebra sampling samples a vector from the outer products uniformly, we sample each chamber uniformly.
\end{proof}

\section{Geometric Algebra sampling for Neural Networks with Bias}
For neural network models with a bias term $f^\text{ReLU}_{\theta,b}(X)$, the convex optimization formulation of the neural network training problem follows:
\begin{equation}\label{cvxnn:relu_bias}
\begin{aligned}
    \min_{\{(u_i,u_i',b_i,b_i')\}_{i=1}^q}&\; \ell\pp{\sum_{i=1}^qD_i^\text{b}X(u_i-u_i'),y}
    +\beta\sum_{i=1}^q(\|u_i\|_2+\|u_i'\|_2)\\
    \text{ s.t. }&(2D_i^\text{b}-I)(Xu_i+bi\bone)\geq 0, \\
    &(2D_i^\text{b}-I)(Xu_i'+b_i'\bone)\geq 0,i\in[q].
\end{aligned}
\end{equation}
Here $D_1^\text{b},\dots,D_q^\text{b}$ are enumeration of all possible hyperplane arrangements $\{\diag(\mbI(Xw+b\bone\geq 0))|w\in\mbR^d,b\in\mbR\}$. For the gated ReLU neural networks with the bias term, the convex optimization formulation takes the form:
\begin{equation}\label{cvxnn:grelu_bias}
\begin{aligned}
    \min_{\{(u_i,b_i)\}_{i=1}^q}&\; \ell\pp{\sum_{i=1}^qD_i^\text{b}(Xu_i+b_i\bone),y}+\beta\sum_{i=1}^q(\|u_i\|_2).\\
\end{aligned}
\end{equation}
To efficiently approximate the solution of \ref{cvxnn:grelu_bias}, for Gaussian sampling, we subsample $\bar D_i^\text{b} = \diag(\mbI(Xv_i+b_i\bone\geq 0))$ for $i\in[k]$, where $v_1,\dots,v_k$ are i.i.d. random vectors following $\mcN(0,I)$ and $b_1,\dots,b_k$ are i.i.d. random variables following $\mcN(0,1)$. 
\begin{equation}\label{cvxnn:grelu_sample_bias}
    \min_{\{(u_i,b_i)\}_{i=1}^k}\; \ell\pp{\sum_{i=1}^k \bar D_i^\text{b}Xu_i+b_i\bone,y}+\beta\sum_{i=1}^k\|u_i\|_2.
\end{equation}


\begin{algorithm}[!htp]
\caption{Convex neural network training with the bias term via Gaussian sampling}
\label{alg:Gaussian_bias}
\begin{algorithmic}[1]
\REQUIRE Number of hyperplane arrangement samples $k$, regularization parameter $\beta>0$.
\STATE Sample $k$ i.i.d. random vectors $v_1,\dots,v_k$ following $\mcN(0,I)$. Sample $k$ i.i.d. random variables $b_1,\dots,b_k$ following $\mcN(0,1)$. 
\STATE Compute $\bar D_i^\text{b} = \diag(\mbI(Xv_i+b_i\geq 0))$ for $i\in[k]$.
\STATE Solve the convex optimizatgion problem \eqref{cvxnn:grelu_sample_bias}.
\end{algorithmic}
\end{algorithm}

\begin{algorithm}[!t]
\caption{Convex neural network training with the bias term via randomized Geometric Algebra}
\label{alg:GA_bias}
\begin{algorithmic}[1]
\REQUIRE Number of hyperplane arrangement samples $k$, regularization parameter $\beta>0$, sketching matrix $S\in\mbR^{m\times d}$.
\FOR{$i=1,\dots,k$}
\STATE Sample $\{j_i\}_{i=1}^{r}$ from $[n]$.
\STATE Compute $\tilde v$ and $b$ via \eqref{optimal:sketch_vb} and let $v=S^T\tilde v$.
\STATE Compute $\bar D_i = \diag(\mbI(X v+b\bone \geq 0))$.
\ENDFOR
\STATE Solve the convex optimization problem \eqref{cvxnn:grelu_sample_bias}.
\end{algorithmic}
\end{algorithm}

From \cite{pilanci2023complexity}, the optimal neurons are given by
\begin{equation}\label{optimal:vb}
    v=\times(x_{j_1}-x_{j_d},\dots,x_{j_{d-1}}-x_{j_d}), b = -v^Tx_{j,d},
\end{equation}
where $\{j_i\}_{i=1}^d$ is a subset of $[n]$. Therefore, for Geometric algebra sampling, we can sample a size-$d$ subset of $[n]$ and compute $D^b=\diag(\mbI(Xv+b\bone\geq 0))$ with $v,b$ computed in \eqref{optimal:vb}. 

To apply the randomized embeddings for neural network models with a bias term $f^\text{ReLU}_{\theta,b}(X)$, we can compute the optimal neuron for the projected data as follows:
\begin{equation}\label{optimal:sketch_vb}
    \tilde v=\times(S(x_{j_1}-x_{j_r}),\dots,S(x_{j_{r-1}}-x_{j_r})), b = -v^TSx_{j,d}.
\end{equation}
Then, we can embed $\tilde v\in\mbR^r$ to $\mbR^d$ by $v=S^T\tilde v$. Then, we can compute $D=\diag(\mbI(Xv+b\geq 0))$ as a hyperplane arrangement. The overall algorithm is summarized in Algorithm \ref{alg:GA_bias}.

\end{document}